\documentclass[a4paper, reqno]{amsart} 

\usepackage{amsmath,amsthm,amssymb,amsfonts,mathrsfs,color,hyperref, mathtools,crop, graphicx, enumitem, todonotes}
\usepackage[]{geometry}
\usepackage{apptools}

\theoremstyle{plain}
\begingroup
\newtheorem{theorem}{Theorem}
\newtheorem*{theorem*}{Theorem}
\newtheorem{corollary}[theorem]{Corollary}

\newtheorem{lemma}[theorem]{Lemma}
\endgroup

\def\arXiv#1{arXiv:\href{http://arXiv.org/abs/#1}{#1}}

\theoremstyle{definition}
\begingroup
\newtheorem{definition}[theorem]{Definition}
\endgroup

\theoremstyle{remark}
\begingroup
\newtheorem{remark}[theorem]{Remark}
\newtheorem{example}[theorem]{Example}
\endgroup 

\numberwithin{equation}{section}
\numberwithin{theorem}{section}

\newenvironment{pde}{\left\{\begin{array}{rll} } {\end{array}\right.}

\newcommand{\C}{\mathbb C} 
\newcommand{\N}{\mathbb N}

\newcommand{\R}{\mathbb R} 

\renewcommand{\P}{{\mathcal P}}

\newcommand{\M}{{\mathcal M}}

\newcommand{\K}{{\mathbb K}}

\newcommand{\Ra} {\Rightarrow}

\renewcommand{\d}{\mathrm{d}}

\newcommand{\dz}{\,\mathrm{d}z}

\newcommand{\dt}{\,\mathrm{d}t}

\newcommand{\HP}{\mathcal{HP}}

\newcommand{\eps}{\varepsilon}
\newcommand{\average}{{\mathchoice {\kern1ex\vcenter{\hrule height.4pt
width 6pt depth0pt} \kern-9.7pt} {\kern1ex\vcenter{\hrule
height.4pt width 4.3pt depth0pt} \kern-7pt} {} {} }}

\let \n =\nabla

\allowdisplaybreaks

 \makeatletter
\@namedef{subjclassname@2020}{2020 Mathematics Subject Classification}
\makeatother
\newcommand\showlabel{\addtocounter{equation}{1}\tag{\theequation}}

\begin{document}

\title[neural networks with analytic and polynomial activation functions on $\R$ and $\C$]{Qualitative neural network approximation over $\R$ and $\C$:\\
{\small Elementary proofs for analytic and polynomial activation}}

\author{Josiah Park}
\address{Josiah Park\\
Department of Mathematics\\
Texas A\&M University\\
155 Ireland Street\\
College Station, TX 77840
}
\email{j.park@tamu.edu}

\author{Stephan Wojtowytsch}
\address{Stephan Wojtowytsch\\
Department of Mathematics\\
Texas A\&M University\\
155 Ireland Street\\
College Station, TX 77840
}
\email{stephan@tamu.edu}

\date{\today}

\subjclass[2020]{
68T07, 
41A30, 
41A10, 
32A05
}
\keywords{Uniform approximation theorem, approximation of holomorphic functions, approximation of harmonic functions, holomorphic neural network, complex analytic activation function, real analytic activation function, polynomial activation function, shallow neural network, deep residual network, DenseNet}

\begin{abstract}
In this article, we prove approximation theorems in classes of deep and shallow neural networks with analytic activation functions by elementary arguments. We prove for both real and complex networks with non-polynomial activation that the closure of the class of neural networks coincides with the closure of the space of polynomials. The closure can further be characterized by the Stone-Weierstrass theorem (in the real case) and Mergelyan's theorem (in the complex case). In the real case, we further prove approximation results for networks with higher-dimensional harmonic activation and orthogonally projected linear maps.

We further show that fully connected and residual networks of large depth with polynomial activation functions can approximate any polynomial under certain width requirements. All proofs are entirely elementary.
\end{abstract}

\maketitle


\section{Introduction}

Neural networks are becoming increasingly popular tools in fields outside of the classical domain of data science. Chief among these applications are experiments in scientific computing, which often involve the solution of potentially high-dimensional partial differential equations. In many important problems -- e.g.\ computations involving quantum systems with many particles -- the output of the neural network may be complex valued. For such problems, the approximation power of complex neural networks with activation functions $\sigma:\C\to\C$ has recently come under investigation\cite{gao2021complex,trabelsi1705deep,tygert2016mathematical,virtue2017better}, and sufficient conditions were established for $\sigma$ under which (shallow, deep) neural networks can approximate {\em any} continuous function on a compact subset of $\C^d$ \cite{voigtlaender2020universal}.

It is a well-known fact that the field of complex numbers and the algebra of complex-differentiable functions behave in surprising ways when compared to the analogous objects of real analysis. For instance, if $D\subseteq \C$ is a bounded open set, the space of functions $f:\overline D\to\C$ which are complex differentiable in $D$ and continuous on $\overline D$ form a proper closed subspace of the space of continuous functions on $\overline D$. In particular there exists a continuous function $f:\overline D\to \C$ such that $\|f- p\|_{C^0(\overline D)} \geq 1$ for all polynomials $p$. This contrasts drastically with the situation in real analysis, where polynomials are dense in the space of continuous functions on a compact set as established by the Stone-Weierstrass theorem.

A similar observation can be made in classes of neural networks. For a function $\sigma:\C\to \C$ and a set of parameters
\[
\Theta:= \big\{(a_k, w_k, b_k)\in \C\times \C^d\times \C : k=1,\dots, n\},
\]
we define
\[\showlabel \label{eq shallow neural network}
f_\Theta(z) = \sum_{k=1}^n a_k \,\sigma\big(\langle w_k,z\rangle + b_k\big)
\]
where $z= (z_1,\dots,z_d) \in \C^d$ is a complex vector and $\langle w,z\rangle = \sum_{i=1}^d \overline w_i\cdot z_i$. We denote the set of all such functions with a fixed number $n$ of `neurons' by $\M_{\sigma,n}$ and $\M_{\sigma} = \bigcup_{n=1}^\infty \M_{\sigma,n}$. The following observations are immediate:

\begin{enumerate}
\item If $\sigma$ is a polynomial in of degree $m$ in one complex variable, then $f_\Theta$ is a polynomial of degree $m$ in $d$ complex variables.
\item If $\sigma$ can be represented by a convergent power series on $\C$, then $f_\Theta$ can be represented by a convergent power-series in $d$ complex variables.
\end{enumerate}
In particular, if $D\subseteq\C$ is open, there exists a continuous function $f:\overline D\to \C$ which {\em cannot} be approximated by functions of the form \eqref{eq shallow neural network}. Depending on the application, this `rigidity' may be an asset or an obstacle. Encoding a priori information about the solution of a problem (often referred to as `domain knowledge') in the neural network architecture has proved invaluable in many tasks, for example by designing convolutional neural networks to approximately respect translation invariance \cite{o2015introduction}, using periodic activation functions in signal processing \cite{alderighi1996advanced,sitzmann2020implicit,xuan2021low,xuan2021deep}, designing specialized neural networks for data in hyperbolic spaces \cite{chami2019hyperbolic,ganea2018hyperbolic,liu2019hyperbolic,peng2021hyperbolic} or directly enforcing physical symmetries in computational chemistry \cite{Zhang:2018va, Chen_2020}. By analogy, if we can show that the solution to a problem in scientific computing is given by a holomorphic function or operator, it serves us well to encode this into the design of our neural network.

The question remains: When trying to approximate a holomorphic function $f: D\subseteq \C^d\to \C$, which activation functions can be used? We present an approximation theorem which treats complex and real shallow neural networks in a unified fashion.

\begin{theorem}\label{theorem main 1}
Let $\K\in \{\R,\C\}$ and let
\[
\sigma:\K\to\K, \qquad \sigma(z) = \sum_{k=0}^\infty \alpha_kz^k
\]
be an analytic function defined by a power series with infinite radius of convergence. Consider the class of shallow neural networks of arbitrary finite width $n$ with activation $\sigma$
\[
\M_\sigma = \bigcup_{n=1}^\infty \M_{\sigma,n}, \qquad \M_{\sigma, n} = \left\{\sum_{k=1}^n a_k \,\sigma\big(\langle w_k,z\rangle + b_k\big) : (a_k, w_k, b_k) \in \K\times \K^d\times \K\right\}.
\]
Let $D\subseteq \K^d$ be an open bounded subset and denote by $C^0(\overline D)$ the space of continuous functions from $\overline D$ to $\K$.
\begin{enumerate}
\item If $\sigma$ is a polynomial of degree $m$ in $z$, then $\M_\sigma$ is the space $\P_m$ of polynomials of degree $m$ in $z$.
\item If $\sigma$ is a not a polynomial, then the closures of $\M_\sigma$  and $\P:= \bigcup_{m=0}^\infty \P_m$ in $C^0(\overline D)$ coincide.
\end{enumerate}
\end{theorem}

If $\K = \R$, we can now recover the classical universal approximation theorem for networks with analytic activation by appealing the Stone-Weierstrass theorem \cite[Section 15.7]{MR2374633}:

\begin{theorem*}[Stone-Weierstrass Theorem]
Let $K\subset \R^d$ be compact and $\mathcal P$ the vector space of polynomials. Then $\mathcal P$ is dense in $C^0(K)$.
\end{theorem*}

While the result is restrictive due to the strong assumptions on $\sigma$, our proof is entirely elementary and does not require advanced techniques beyond an introductory class in (real, complex) analysis.

If $\K=\C$, we emphasize that $\sigma$ is analytic in $z$, not $(z,\bar z)$ or $(x,y)$, and that the elements of $\P_m$ are equally polynomials in the complex variable $z$. The closure of $\P$ depends on the topology of the set $D$. If $D= \{z\in \C : r < |z| < R\}$ is an annular domain and $k\in \N$, then the function $f(z) = z^{-k}$ is holomorphic on $U$, but cannot be approximated by polynomials. This follows from Cauchy's integral formula \cite[Theorem II.3.2]{MR1250380} as
\[
\int_\gamma z^{k-1}f(z)\cdot \dz = 2\pi i, \qquad \int_\gamma z^{k-1} \,p(z)\cdot \dz = 0
\]
for all polynomials $p$ and all curves $\gamma$ in $D$ which loop around the origin. If there were polynomials which could approximate $f$ uniformly, also the integrals would have to converge. 

On the positive side, in 1951 Mergelyan \cite{MR0041929} showed the following \cite[Theorem 20.5]{MR924157}: 

\begin{theorem*}[Mergelyan's Theorem]
 Let $K\subseteq \C$ be compact such that $\C\setminus K$ is connected, and $f:K\to \C$ a continuous function which is holomorphic in the interior $K^\circ$ of $K$. Then for every $\eps>0$, there exists a complex polynomial $P$ such that $\sup_{z\in K}|f(z) - p(z)| < \eps$.
 \end{theorem*}
 
 The class of `good' domains in particular includes all simply connected bounded open sets with Lipschitz boundary.
The situation in many complex variables is more complicated and not entirely understood. While multi-variate Mergelyan-type theorems \cite{MR4215751,MR3450139,cho1998mergelyan} and related results like the Oka-Weil Theorem \cite{MR0132202,MR1512987} have been obtained, there are obstructions to proving the statement in full generality. Notably, Diederich and Fornaess \cite{MR397019} constructed an example of a pseudoconvex domain $D\subseteq \C^2$ with smooth boundary and a continuous function $f:\overline D\to \C$ such that $f$ is holomorphic in $D$, but cannot be approximated uniformly by polynomials in $\overline D$. Thus even for functions of multiple complex variables which can be shown to be holomorphic on a `good' domain, there may be deep obstructions to approximation by both polynomials and holomorphic neural networks. A recent survey of holomorphic approximation can be found e.g.\ in \cite{MR4264040}.

While holomorphic functions can be thought of simultaneously as a generalization of differentiable functions and infinitely differentiable functions, the perhaps closest analogue in real analysis is the class of {\em harmonic} functions. In two dimensions, a correspondence between harmonic and holomorphic functions on $\C=\R^2$ can be constructed by taking the real part of a function, which is one-to-one up to an affine shift in the imaginary part. Like holomorphic functions, harmonic functions in any dimension form a closed proper subspace of $C^0(\overline D)$ for open bounded $D\subseteq\R^d$. For a deeper understanding of complex-analytic rigidity, we study real neural networks with harmonic activation functions. 

Since harmonic functions in one real dimension are just affine linear, the interesting case concerns activation functions $\sigma$ of two or more real variables. Furthermore, we restrict the linear representation of data to be angle-preserving projections. Namely, if $\sigma:\R^k\to \R$ is harmonic, we consider a class of functions on $\R^d$ for $d\geq k$ given by $ \M_\sigma = \bigcup_{n=1}^\infty \M_{\sigma,n}$ where
\[
\M_{\sigma,n} = \left\{\sum_{i=1}^n a_i \sigma\big(\rho_i\,P_ix + b_i\big) : \rho_i\in \R,\: b_i \in \R^k, \: P_i\in \R^{k\times d} \text{ s.t. }P_iP_i^T = I_{k\times k}\right\}.
\]
Geometrically, the linear maps $P_i$ are orthogonal projections from $\R^d$ to $\R^k$ for all $i$. It is easy to see that $f\in\M_\sigma$ is harmonic, and thus that $\M_{\sigma}$ is at most dense in the space of harmonic functions. We show the following.

\begin{theorem}\label{theorem main harmonic}
Let $\sigma:\R^k\to\R$ be a harmonic function, $d\geq k$ and $D\subseteq \R^d$ open and bounded. 

\begin{enumerate}
\item If $\sigma$ is a harmonic polynomial of degree $m$, then $\M_\sigma$ is the class of harmonic polynomials of degree $m$
\[
\mathcal{HP}_m (\R^d) = \{p\in \P_m (\R^d): \Delta \,p = 0\}.
\]
\item If $\sigma$ is not a polynomial, then the closures of $\M_\sigma$ and $\mathcal{HP} = \bigcup_{m=0}^\infty\mathcal{HP}_m$ in $C^0(\overline D)$ coincide.
\end{enumerate}
\end{theorem}

So far, we only considered {\em shallow} neural networks with two layers (i.e.\ one hidden layer). While this case is historically well-studied, modern neural networks are `deep', i.e.\ they have many layers. 
To keep things simple, we focus on functions $f:\K^d\to\K$ which are represented by residual neural networks (ResNets), but comparable results can be obtained for classical fully connected feedforward networks and DenseNets (see Appendix \ref{appendix fully connected}). Residual neural networks form a function class which is comparable to feed forward networks in terms of approximation power, but with a parametrization that facilitates gradient-based optimization to find appropriate network parameters. The incremental nature of the change to the internal state in every layer alleviates the {\em vanishing and exploding gradients phenomenon}, which is the main motivation of the ResNets in \cite{he2016deep}. For this reason, truly `deep' networks typically have a form of residual structure. Continuum limits for infinitely deep neural networks have been studied in \cite{weinan2017proposal,li2017maximum} and later in \cite{chen2018neural} as `neural ODEs'.

 A ResNet can be understood as follows.

\begin{itemize}
\item Let $d_0\in \N$. For a given input $z\in \overline D$ and parameters $A^0\in \K^{d_0\times d}$, $b^0\in \K^{d_0}$, designate $z^0 = A^0 z + b^0 \in \K^{d_0}$.
\item For $\ell \in \{1, \dots, L-1\}$ and parameters $A^\ell \in \K^{d_0 \times D_\ell}, W^\ell \in \K^{D_\ell\times d_{0}}, b^\ell \in \K^{D_\ell}$ for some $D_\ell \in \N$, set 
\[\showlabel \label{eq resnet}
z^\ell = z^{\ell-1} + A^\ell\,\sigma( W^\ell z^{\ell-1} + b^\ell)
\]
where $\sigma$ is applied coordinatewise.
\item For parameters $A^L\in \K^{d_0}$, set $z^L = A^L z^{L-1}$.
\end{itemize}
As previously, we collect the weights in a single vector
\[
\Theta = \big(A^0, b^0, A^1, b^1, W^1, \dots, A^{L-1}, b^{L-1}, W^{L-1}, A^L\big)
\]
and denote $f_\Theta(z) = z^L$, where we suppressed the dependence of $z^L$ on the input $x$ and weights $\Theta$ for the sake of compact notation. It should be noted that sometimes in ResNets, the architecture is specified further by taking $D_\ell \equiv d_0$ and $A^\ell$ as the unit matrix.

Deeper neural networks have multiple parameters that govern their complexity: The depth $L$ and the vector of widths 
$d^0, D^1,\dots, D^{L-1}$. The width of the input layer $d$ and the output layer $d_L=1$ are given by the problem statement. The notation therefore becomes somewhat less compact compared to shallow networks. We denote the classes of ResNets with a fixed architecture by $\mathcal R_{L,\sigma}(d_0, D_1,\dots, D_{L-1})$. We can now present our third main result.

\begin{theorem}\label{theorem main 2}
Let $\K\in \{\R,\C\}$ and let
\[
\sigma:\K\to\K, \qquad \sigma(z) = \sum_{n=0}^\infty \alpha_nz^n
\]
be an analytic function defined by a power series with infinite radius of convergence. Let $U\subseteq \K^d$ be an open bounded subset.
\begin{enumerate}
\item Let $L\geq 2$ be fixed, $d_0\geq d+1$ and $n:= D_1 + \dots + D_{L-1}$. Then $\M_{\sigma,n} \subseteq \mathcal R_{L,\sigma}(d_0, D_1,\dots, D_{L-1})$. In particular, if $\sigma$ is not a polynomial, then the closures of 
\[
\mathcal N_1:= \bigcup_{D_1 = 1}^\infty \mathcal R_{L,\sigma}(d+1, D_1, 0, \dots, 0), \qquad \mathcal N_2 := \bigcup_{L = 1}^\infty \mathcal R_{L,\sigma}(d+1, m, m, \dots, m)
\]
and the closure of the space of polynomials $\mathcal P$ in $C^0(\overline D)$ coincide for all $m\geq 1$. 

\item If $\sigma$ is not a linear function, then the closure of 
\[
\widetilde{\mathcal N_2} := \bigcup_{L = 1}^\infty \mathcal R_{L,\sigma}(d+2,m,\dots,m)
\]
and the closure of the space of polynomials $\mathcal P$ in $C^0(\overline D)$ coincide for all $m\geq 4$.
\end{enumerate}
\end{theorem}

Roughly speaking, we only need to have the width of one layer go to infinity, while all others remain bounded, or the depth go to infinity if $d_0\geq d+1$. This is not surprising: In the first case, we approximate the target function with the first hidden layer and the remaining layers are the identity map, while in the second case, we `turn the shallow neural network on its side'. This idea is classical and will be explained in greater detail below for the reader's convenience. The restriction that $d_0\geq d+1$ is equally classical if the input dimension is $d\geq 2$, as there are obstructions to universal approximation for thinner networks \cite{johnson2018deep}. 

The most interesting of the results given in Theorem \ref{theorem main 2} is the fact that polynomial activation functions are admissible if the neural network is slightly wider than before and we vary the depth rather than the width. This is not entirely surprising since elements of $\mathcal R_{L,\sigma}(d_0, D_1,\dots, D_{L-1})$ are polynomials of degree at most $m^L$ if $\sigma$ is a polynomial of degree $m$. As $L\to\infty$, the upper bound becomes less and less restrictive. Nevertheless, it is not immediately obvious that we can in fact approximate {\em all} polynomials. We believe the result to be folklore, but are not aware of a rigorous proof in the literature.

The approximation properties of neural networks have been studied in great detail by many authors over the course of at least three decades in 
\cite{cybenko1989approximation,hornik1991approximation, leshno1993multilayer,MR1237720, makovoz1998uniform,pinkus1999approximation,maiorov1999lower,arora2016understanding,eldan2016power,yarotsky2017error, MR3856963,klusowski2018approximation,deep_barron, MR3564936,
klusowski2018approximation,petersen2018optimal, siegel2019approximation,yarotsky2019phase, chen2019efficient,schmidt2019deep,parhi2021banach,parhi2021kinds,siegel2020high,siegel2020approximation,devore2021neural,daubechies2021neural,daubechies2022nonlinear, gribonval2022approximation}
to name only a few. The main goal of this article is to provide a self-contained and elementary introduction to qualitative versions of universal and qualified approximation theorems, not to improve upon the state of the art in specific classes of functions. As such, the proofs in the remainder of the article are elementary and require little or no knowledge beyond undergraduate real and complex analysis. Nevertheless, we maintain that several results, including the unified treatment of real and complex networks, as well as the treatment of polynomial activation functions, are at most folklore to the best of our knowledge. The main novel contribution of these notes, the proof of Theorem \ref{theorem main harmonic} in Section \ref{section proof harmonic}, is the only place where deeper results are used.

\section{Proof of Theorem \ref{theorem main 1}: Shallow networks}\label{section proof main 1}

In the real case, a version of this proof goes back to \cite{mhaskar1996neural}.

\begin{proof}
{\bf Step 1.} Assume that $\sigma(z) = \sum_{j=0}^\infty \alpha_jz^j$. In this step, we show that for any $m\in\N_0$, the function $z\mapsto \alpha_mz^m$ can be approximated by a shallow neural network $f_m\in\M_{\sigma,m+1}$ of the form \eqref{eq shallow neural network} in one dimension. This holds trivially in the case $m=0$, since $\sigma(0) = \alpha_0z^0$ and $z\mapsto \sigma(0)= 1\cdot\sigma(\langle 0,z\rangle + 0)\in \M_{\sigma,1}$. For $m\geq 1$, note that
\[
\frac{d^m}{dh^m}\bigg|_{h=0} \sigma(hz) = \frac{d^m}{dh^m}\bigg|_{h=0}  \sum_{j=0}^\infty a_j(hz)^j =m!\, \alpha_mz^m
\]
since power series and their derivatives converge locally uniformly, so summation and differentiation commute \cite[Sections 6.4 and 15.2]{MR2374633}. The $m$-th derivative of $\sigma$ is the limit of iterated difference quotients
\[
\frac{d^m}{dh^m}\sigma(hz) = \lim_{\gamma\to 0} \frac{\sum_{l=0}^m(-1)^l\binom{m}l \sigma\big((h+l\gamma)z\big)}{\gamma^m}\ \text{ thus }\ \lim_{\gamma\to 0} \frac{\sum_{l=0}^m(-1)^l\binom{m}l\, \sigma\big(l\gamma z\big)}{\gamma^m} = m!\,\alpha_mz^m,
\]
where the limit holds uniformly in the set $|z|\leq R$ for any $R>0$.\footnote{ This is a corollary to either the Mean Value Theorem or the Fundamental Theorem of Calculus in real analysis, which easily applies to the complex case. Strangely, it seems to be omitted in many textbooks on real analysis. A reference in a somewhat more general setting can be found in \cite[Corollary 4.4]{MR1216137}.} Since
\[
z\mapsto \frac1{\gamma^m}\sum_{l=0}^m (-1)^l\binom{m}l\, \sigma\big((h+l\gamma)z\big) \in \M_{\sigma,m+1}
\]
by definition, the result is proved. 

{\bf Step 2.} In this step, we show that if $a_M\neq 0$, then $z\mapsto z^{m}$ can be approximated uniformly by elements of $\M_\sigma$ for any $m\leq M$. The result is trivial for $m=M$, since we can divide by $a_M$. For $m<M$, the result follows as previously by noting that
\begin{align*}
z^{m} &= \frac{m!}{M!} \,\frac{d^{M-m}}{dz^{M-m}} z^M = \frac{m!}{M!} \,\lim_{\beta\to 0} \frac{\sum_{n=0}^{M-m}(-1)^n\binom{M-m}n (z+n\beta)^M}{\beta^{M-m}} \\
	&= \frac{m!}{a_M\cdot M!} \,\lim_{\beta\to 0} \lim_{\gamma\to 0}\frac{\sum_{n=0}^{M-m}\sum_{l=0}^M(-1)^{n+l}\binom{M}l\binom{M-m}n \,\sigma(l\gamma z+n\beta)}{\beta^{M-m}\,\gamma^M},
\end{align*}
and noting that the approximating functions are elements of $\M_{\sigma, 2M-m+2}$.

{\bf Step 3.} From now on, we will consider the general case $d\geq 1$. Assume that $\sigma$ is a polynomial of degree $m$. Then clearly $\M_\sigma \subseteq\P_m$ since 
\[
\sum_{k=1}^n a_k\,\sigma(\langle w_k,z\rangle+b_k) = \sum_{k=1}^n a_k\sum_{j=0}^m \alpha_j \big(\langle w_k,z\rangle+ b_k\big)^j
\]
is a polynomial of degree $m$ in $d$ variables. On the other hand, we claim that any polynomial $p$ of degree $m$ can be approximated uniformly by elements of $\M_\sigma$. It suffices to consider the case that 
\[
p(z) = z_1^{m_1}\dots z_d^{m_d}
\]
is a monomial of degree $\overline m := m_1 + \dots + m_d\leq m$. Then $p\in \overline{\M_\sigma}$ by the same rationale as before since
\[
p(z) = \frac1{\overline m !} \,\frac{\partial^{\overline m }}{\partial_{h_1}^{m_1}\dots \partial_{h_d}^{m_d}} \big(h_1z_1+\dots + h_dz_d\big)^{\overline m }.
\]
By considering one-dimensional slices and Step 2, the function
\[
\big(h_1z_1+\dots + h_dz_d\big)^{\overline m }
\]
lies in $\overline{\M_\sigma}$ for any choice of $h_1,\dots, h_d$ and $\overline m $. As previously, also its difference quotients in $h$ can be approximated by elements of $\M_\sigma$. Since monomials can be approximated and $\M_\sigma$ is a linear class, we find that $\P_m\subseteq \overline{\M_\sigma}$.

{\bf Step 4.} If $d\geq 1$ and $\sigma$ is not a polynomial, then for every $m\in \N_0$, there exists $M\geq m$ such that $a_M\neq 0$. As in step 3, we can show that every polynomial of degree at most $m$ can be approximated uniformly by elements $f_m\in \M_{\sigma}$. Taking the union over $m\in \N$, we find that any polynomial can be approximated uniformly by $f\in \M_\sigma$. By a standard diagonal sequence argument, we conclude that $\overline{\bigcup_{m=0}^\infty \P_m}\subseteq\overline{\M_\sigma}$.

On the other hand, let $f\in \subseteq\overline{\M_\sigma}$. Then, for every $\eps>0$, there exist $n\in \N$ and $f_n\in \M_{\sigma,n}$ such that $\|f_n- f\|_{C^0(\overline D)}<\eps/2$. Since 
\[
f_n(z) =  \sum_{k=1}^n a_k \,\sigma(\langle w_i,z\rangle + b_i) = \lim_{m\to \infty} \sum_{j=1}^m \sum_{i=1}^n \alpha_ka_i \,\big(\langle w_i,z\rangle + b_i\big)^m
\]
uniformly in the set $|z|\leq R$ for any given $R>0$, we can truncate the series for $f_n$ at an index $m$ such that 
\[
\|f_{n,m}- f\|_{C^0(\overline D)} \leq \|f_{n}- f\|_{C^0(\overline D)} + \|f_{n}- f_{n,m}\|_{C^0(\overline D)} <\eps.
\]
As this can be done for any $\eps>0$, we have $f\in \overline{\bigcup_{m=0}^\infty \P_m}$.
\end{proof}

Since power series converge to their limit in $C^k$ for any $k\in \N$,\footnote{ Since their derivatives, which are also power series, converge uniformly.} we find that $\overline{\M_\sigma}\subseteq \overline{\P}$ also if the closure is taken with respect to the $C^k$-topology. Conversely, since the $m$-th difference quotient of a $C^{m+k}$-function also converges in $C^k$, and since $\sigma\in C^\infty$, we can use the proof of Theorem \ref{theorem main 1} also to see that $\overline{P} \subseteq\overline{\M_\sigma}$ in the $C^k$-topology.

\begin{corollary}\label{corollary ck approximation}
Let $\K\in \{\R,\C\}$ and let
\[
\sigma:\K\to\K, \qquad \sigma(z) = \sum_{n=0}^\infty \alpha_nz^n
\]
be an analytic function defined by a power series with infinite radius of convergence. Consider the class $\M_\sigma$ of shallow neural networks of arbitrary finite width $n$ and activation $\sigma$ as before. Let $D\subseteq \K^d$ be an open bounded subset, $k\geq 1$, and denote by $C^k(\overline D)$ the space of $k$ times differentiable functions from $D$ to $\K$ such that the derivatives of all orders extend continuously to the closure $\overline D$.
\begin{enumerate}
\item If $\sigma$ is a polynomial of degree $m$, then $\M_\sigma$ is the space $\P_m$ of polynomials of degree $m$.
\item If $\sigma$ is a not a polynomial, then the closures of $\M_\sigma$  and $\P:= \bigcup_{m=0}^\infty \P_m$ in $C^k(\overline D)$ coincide.
\end{enumerate}
\end{corollary}

In Corollary \ref{corollary ck approximation}, we replaced the $C^0$-topology by the stronger $C^k$-topology. Similarly, we could pass from a stronger topology like $C^0$ or $C^k$ to a weaker one, like $L^p$ or $W^{k,p}$ and conclude that the closures of $\P$ and $\M_\sigma$ coincide.

There exists a finite number of neurons $n$ that a shallow neural network of the form \eqref{eq shallow neural network} requires to approximate the function $z\mapsto z^m$ to {\em arbitrary} accuracy. Reaching higher precision requires increasing the magnitude of weights, but not  the number of neurons. In particular, $n$ depends only on $m$ and on which coefficients of $\sigma$ in a power series expansion are non-zero. It seems advantageous to choose activation functions in which all coefficients are non-zero (such as $\exp$), or at least such that there are no long gaps in the set of non-zero coefficients (such as $\sin, \cos, \sinh, \cosh$). If {\em all} coefficients are non-zero, then the proof of Theorem \ref{theorem main 1} also illustrates that the bias term $b_i$ in $\langle w_i, z\rangle + b_i$ is not needed to prove approximation results. This explains the density of Fourier series in the space of continuous functions, which are formally neural networks with a single hidden layer and the activation function $\sigma(z) = \exp(2\pi i z)$. In a Fourier series, all biases are set to zero.

The conditions on $\sigma$ in Theorem \ref{theorem main 1} can be weakened somewhat. To approximate $z^m$ by elements of $\M_\sigma$, we `zoomed in' suitably at the origin to utilize the power series expansion. In particular, if $f$ can be represented by a convergent non-polyomial power series in a neighbourhood of the origin on the real line, then every polynomial can be approximated arbitrarily well by elements of $\M_\sigma$. This applies in particular to real analytic activation functions like $\tanh(x) = \frac{e^x-e^{-x}}{e^x+e^{-x}}$ or the sigmoid function $\sigma(x) = \frac{1}{1+e^{-x}}$, whose power series representation does not converge globally.
Since polynomials are dense in the space of continuous functions due to the Stone-Weierstrass Theorem \cite[Section 15.7]{MR2374633}, it follows that $\overline{\M_\sigma} = \overline{\P} = C^0(\overline D)$, if the closure is taken in the uniform topology.

In the complex plane, the assumption that the radius of convergence of $\sigma$ is $\infty$ is used to prove that elements of $\M_\sigma$ can be approximated by polynomials. It is implied by the assumption that $\sigma$ is holomorphic on the entire plane $\C$, since the radius of convergence can be characterized as the distance to the closest singularity (which, in this case, is infinity) \cite[Theorem 16.2]{MR924157}. If $\K=\C$ and $\sigma$ is a rational function, then the closure of $\M_\sigma$ may be strictly larger than that of the space of polynomials (e.g.\ if $U$ is annular domain and the weights of the network are chosen such that a pole of $\sigma$ is inside the hole in $U$). In this situation `neural networks' $f:\C\to\C$ with activation $\sigma(z) = z^{-1}$ can approximate any meromorphic functions $f:\C\to \overline \C$ in $C^0(K)$, if $f$ does not have a singularity in $K$. The key observation is that any type of pole $z\mapsto z^{-m}$ can be generated by a superposition of derivatives of $\sigma$. Difference quotients approximate these derivatives uniformly away from the singularity. 

Of course, if $\sigma$ is a rational function (and not a polynomial), then there exists $z^*\in \C$ such that $\lim_{z\to z^*}|f(z)| =\infty$. When approximating a target function which is bounded on $\overline D$, the weights should be chosen such that these poles lie outside of $\overline D$. However, especially in the initial phase of training, the infinite gradients of $f$ may lead to greater numerical instability, and there is no guarantee that the domain $U$ is captured accurately by finite amounts of data. The greater expressivity therefore comes with a not so hidden cost.

In the context of machine learning, obstructions to polynomial approximation (in the complex case) would not be immediately visible: If $\{z_1,\dots, z_N\} \subset\C$ and $\{y_1,\dots, y_N\} \subset \C$ are finite data sets, then there always exists a unique polynomial $P_N$ of degree $N$ such that $P_N(z_i) = y_i$ (assuming that all $z_i$ are different). In Lagrange representation, we can write
\[
P_N(z) = \sum_{i=1}^N y_i \,\prod_{j\neq i} \frac{z- z_j}{z_i-z_j}.
\]
An exact interpolant can also be found in $\M_{\sigma, n}$ for sufficiently large $n$, which may scale linearly with $N$ (depending on which power series coefficients of $\sigma$ vanish). Even for functions which are holomorphic on the data domain $D$, it is therefore imperative to understand whether $f$ can be approximated by polynomials, as this cannot be determined from a finite data set. At most we may notice that $P_N$ does not approximate the function we expected at previously unseen data points (the test set).

As another consequence of Cauchy's integral formula, we recall the following Liouville theorem.

\begin{theorem*}[Liouville's theorem]
If 
\[
\limsup_{r\to \infty} \frac{\max\{|\sigma (z)| : z\in B_r(0)\}}{r^\alpha} < \infty
\]
for some $\alpha>0$, then $\sigma$ is a polynomial of degree at most $\lfloor \alpha\rfloor$ since the $\lfloor \alpha\rfloor+1$-th derivative of $\sigma$ vanishes.
\end{theorem*}

This version of Liouville's Theorem can be proved by appealing to Cauchy's integral formula in a fashion virtually identical to the classical case $\alpha=0$, which states that every bounded holomorphic function is constant. Surprisingly, it is skipped in many standard texts on complex analysis. For a reference, see e.g.\ \cite[Exercise 7.11]{MR1975725}.

Consequently, any holomorphic activation function which generates universal approximators (in the class of holomorphic functions) fails to be Lipschitz continuous. This lack of quantitative global continuity has undesirable consequences from the perspective of statistical learning and gradient-based optimization.

In analogy to \cite{voigtlaender2020universal}, rather than the closure in the $C^0(K)$-topology for a fixed compact set $K$, we can consider the closure of $\M_\sigma$ in the topology of {\em locally uniform convergence} (or compact-open topology) on $\K^d$, i.e.\ the set of functions $f:\K^d\to\K$ such that there exists a sequence $f_n\in \M_\sigma$ such that $f_n\to f$ uniformly on every compact set $K\subset \K^d$. The relationship between the closures is more subtle in our case and discussed in Appendix \ref{appendix compact open}.

\section{Proof of Theorem \ref{theorem main harmonic}: Harmonic shallow networks}\label{section proof harmonic}

\subsection{A primer on harmonic polynomials}

Before we come to the main proof, we review some properties of harmonic functions and harmonic polynomials.

\begin{definition}
\begin{enumerate}
\item A function $f:\R^d\to\R$ is called {\em harmonic} if $\Delta f = (\partial_1^2+\dots\partial_d^2)f = 0$.
\item A function $f:\R^d\to\R$ is called {\em homogeneous of degree $j$} if $f(\lambda x) = \lambda^jf(x)$ for all $\lambda\in\R$ and $x\in \R^d$.
\end{enumerate}
\end{definition}

It is an easy exercise to see that if $p:\R^d\to\R$ is harmonic, then so is $x\mapsto p(Ox)$ for any orthogonal matrix $O$, see also \eqref{eq network is harmonic}. For polynomials, being homogeneous means that there are no lower order terms, and the degree of the polynomial and degree of homogeneity coincide. Every (harmonic) polynomial can be decomposed uniquely as $p = \sum_{j=0}^{\mathrm{deg}\,p} p_j$, where $p_j$ is a homogeneous (harmonic) polynomial of degree $j$. The fact that the $j$-homogeneous part $p_j$ of a harmonic polynomial $p$ is harmonic follows from the fact that $\Delta p_j$ is $j-2$-homogeneous, and that the terms of different homogeneity must vanish separately. We denote 
\[
\HP^d_j = \{p:\R^d\to\R : p\text{ harmonic homogeneous polynomial of degree $j$}\}.
\]
Before we come to the proof of Theorem \ref{theorem main harmonic}, we require the following auxiliary result about harmonic homogeneous polynomials and coordinate rotations. 

\begin{lemma}\label{lemma polynomial rotations}
Let $p \in \HP^d_j$ for some $d, n$ and denote
\[
V_p = \mathrm{span} \{p(Ox) : O \in O(d)\} \subseteq \HP^d_j.
\]
If $p\not\equiv 0$, then $V_p = \HP^d_j$.
\end{lemma}

In particular, Lemma \ref{lemma polynomial rotations} illustrates that there is no substantial difference between harmonic polynomials in many variables and few variables, since e.g.\ $x_1^2-x_2^2$ and its rotations can be used to generate {\em any} homogeneous harmonic polynomial of degree $2$ on a high-dimensional space.

We prove Lemma \ref{lemma polynomial rotations} in Appendix \ref{appendix polynomial rotations}. 

\subsection{Approximation by harmonic neural networks}
If $\sigma:\R^k\to\R$ is harmonic, then any $f\in \M_\sigma$ is harmonic for any $d\geq k$ since the Laplace operator $\Delta$ is linear and
\begin{align*}
\Delta a_i \sigma\big(\rho_i\,P_ix + b_i\big) &= a_i \sum_{j=1}^d \partial_{x_j}\partial_{x_j}\sigma\big(\rho_i\,P_ix + b_i\big)\\
	&= a_i\rho_i \sum_{j=1}^d \partial_{x_j}\left(\sum_{l=1}^kP_{i; lj} (\partial_{y_l}\sigma)\big(\rho_i\,P_ix + b_i\big)\right)\\
	&= a_i\rho_i^2 \sum_{j=1}^d \sum_{l,m=1}^kP_{i; lj}P_{i; mj} (\partial_{y_l}\partial_{y_m}\sigma)\big(\rho_i\,P_ix + b_i\big)\\
	&= a_i\rho_i^2\sum_{l,m=1}^k\left(\sum_{j=1}^dP_{i; lj}P_{i; mj}\right) (\partial_{y_l}\partial_{y_m}\sigma)\big(\rho_i\,P_ix + b_i\big)\\
	&= a_i\rho_i^2\sum_{l,m=1}^k (P_iP_i^T)_{lm} (\partial_{y_l}\partial_{y_m}\sigma)\big(\rho_i\,P_ix + b_i\big)\\
	&= a_i\rho_i^2\sum_{l,m=1}^k \delta_{lm} (\partial_{y_l}\partial_{y_m}\sigma)\big(\rho_i\,P_ix + b_i\big)\\
	&= a_i\rho_i^2 (\Delta \sigma) \big(\rho_iP_ix+b_i) = 0\showlabel\label{eq network is harmonic}.
\end{align*}
 This implies that element of $\M_\sigma$ cannot approximate any function which is not harmonic due to \cite[Theorem 2.8]{MR1814364} in close analogy to the observations for holomorphic functions. 

\begin{proof}[Proof of Theorem \ref{theorem main harmonic}]
{\bf Step 1.}  Let $q\in \HP_m^d$ and $\sigma:\R^k\to \R$ harmonic such that $\sigma$ is not a polynomial of degree at most $m-1$. In this step, we show that $q\in \overline{\M_\sigma}$. 

Note that for any $t\in\R$, the function $x\in \R^k \mapsto \sigma (tx)$ is harmonic by the same argument as \eqref{eq network is harmonic}. Since sums of harmonic functions are harmonic, this means that
\[
x\mapsto \frac{\sigma(tx) - \sigma(0)}t
\]
is harmonic for every $t\neq 0$, and since the uniform limit of harmonic functions is harmonic, also $x\mapsto \frac{d}{dt}\big|_{t=0}\sigma(tx)$ is harmonic. The same is true for iterated difference quotients, and therefore higher order derivatives. Since $\sigma$ is harmonic, it is analytic and can be written as 
\[
\sigma(x) = \sum_{j=0}^\infty p_j(x)
\]
in a neighbourhood of the origin \cite[Section 2.2, Theorem 10]{MR2597943}, where $p_j$ is a homogeneous polynomial of degree $j$. As for holomorphic functions, we find that
\[
\frac{d}{dt}\bigg|_{t=0}\sigma(tx) = \sum_{j=0}^\infty \frac{d}{dt}\bigg|_{t=0} p_j(tx) =\sum_{j=0}^\infty \frac{d}{dt}\bigg|_{t=0} t^j\,p_j(x)  = j!\,p_j(x).
\]
In particular, we see that $p_j$ is harmonic for all $j\in\N$. Since $\sigma$ is not a polynomial of degree at most $m-1$, there exists some $j\geq m$ such that $p_j\not \equiv 0$. By the preceding analysis and the definition of $\M_\sigma$, we find that 
\[
x\mapsto p_j(P x) \in \overline \M_\sigma
\]
for any $P\in O(d,k)$. Thus, fixing $P:\R^d\to \R^k$ as $x\mapsto (x_1,\dots, x_k)$ and setting $\tilde p_j(x) = p_j(Px)$, we find that $x\mapsto \tilde p_j  (Ox)\in \overline\M_\sigma$ for all $O\in O(d)$. 

If $f$ and $g$ can be approximated to arbitrary accuracy by elements of $\M_\sigma$, the same is true for $f+g$ and $\lambda f$ with $\lambda\in \R$, so $\overline{\M_\sigma}$ is a linear space. In the terminology of Lemma \ref{lemma polynomial rotations}, this means that $V_{\tilde p_j}\subseteq\overline{\M_\sigma}$ and thus $\HP^d_j\subseteq \overline{\M_\sigma}$. If $j=m$, then $q\in \overline{\M_\sigma}$ and the proof is concluded. 

If $j>m\geq 0$, we observe that there exists $1\leq \ell\leq k$ such that $\partial_{x^\ell}p(x)$ is not the zero polynomial, since a non-trivial polynomial has a non-trivial gradient. By \cite[Section 2.2]{MR1251736}, we have $\Delta\partial_{x^\ell}p = \partial_{x^\ell}\Delta p =0$. In particular $\partial_{x^\ell}p(x)$ is a homogeneous harmonic polynomial of degree $j-1$ and
\[
\partial_{x^\ell}p(x) = \lim_{h\to 0} \frac{p(x+he_\ell) - p(x)}h \in \overline{\M_\sigma}
\]
by almost the same construction as before. The main difference is that earlier we took the derivative in the scaling factor $\rho$, where now we take the derivative in the bias $b$. After $j-m$ steps, we find $\tilde p\in \HP^d_m \cap \overline{\M_\sigma}$, and the proof can be concluded as before.

{\bf Step 2.} We have seen that, if $\sigma$ is not a harmonic polynomial of degree at most $m-1$, then every harmonic homogeneous polynomial of degree at most $m$ can be approximated by elements of $\M_\sigma$ in the uniform topology on $\overline D$. By linearity, this is also true for every harmonic polynomial of degree at most $m$. We now distinguish two cases:
\begin{enumerate}
\item $\sigma$ is a polynomial of degree $m$. Then $\M_\sigma$ is contained in the space of harmonic polynomials of degree at most $m$, so 
$\M_\sigma = \overline{\M_\sigma} = \oplus_{j=0}^m \HP_j^d$ is the space of harmonic polynomials of degree at most $m$.
\item $\sigma$ is not a polynomial. Then $\oplus_{j=0}^\infty \HP_j^d\subseteq \overline{\M_\sigma}$, where the $\oplus$ denotes the direct sum in which at most finitely many terms are non-zero. By density $\overline{\oplus_{j=0}^\infty \HP_j^d}\subseteq \overline{\M_\sigma}$.
\end{enumerate}

{\bf Step 3.} We now show that $\overline{\M_\sigma}\subseteq\overline{\HP}$. Since $\sigma$ is a harmonic function on the entire space $\R^k$, it can be represented by a globally convergent power series in many variables
\[
\sigma(x) = \sum_{j=0}^\infty p_j(x)
\]
due to \cite[Section 2.2.e]{MR2597943}, where $p_j$ is a harmonic homogeneous polynomial of degree $j$, much like in Step 1. In particular, for any $P \in O(d,k)$ and $\rho, b\in\R$, we see that $x\mapsto \sigma(\rho Px+b)$ is an analytic function. Since power series converge locally uniformly, we can truncate the series at a finite index depending on $\rho, b$ and $\overline D$ to obtain a harmonic polynomial in $d$ variables which uniformly approximates $\sigma(\rho Px+b)$. By linearity, any element in $\M_\sigma$ can be approximated uniformly by harmonic polynomials. By density and selection of a diagonal sequence, this extends to any element in the closure $\overline{\M_\sigma}$.
\end{proof}

\begin{remark}
The restrictive class of linear maps is crucial in this result, since the harmonicity is only preserved due to the orthogonality constraint. If $\widetilde M_{\sigma,n}$ is the class of all maps
\[
f(x) = \sum_{i=1}^n a_i\,\sigma\big(W_ix+b_i\big)
\]
for general linear maps $W_i:\R^n\to \R^k$ and $\widetilde \M_\sigma = \bigcup_{n=1}^\infty \widetilde\M_{\sigma,n}$, then 
\begin{enumerate}
\item If $\sigma$ is a polyomial of degree at most $m$, then $\M_\sigma= \P_m$ is the space of polynomials of degree $m$.
\item If $\sigma$ is not a polynomial, then $\M_\sigma$ is dense in the space of continuous functions.
\end{enumerate}
This holds even for harmonic activation functions. An easy way to see this is to fix a direction $v\in \R^k$ such that $t\mapsto \sigma(tv)$ is not a polynomial (not a polynomial of low degree) and consider the uniform approximation theorem in one variable, e.g. Theorem \ref{theorem main 1} for maps $W_i$ which project to the line $tv$. To see that such a direction exists, observe the following:
\[
\sigma(x) = \sum_{j=0}^\infty p_j(x)
\]
where $p_j$ is a homogeneous harmonic polynomial of degree $j$. For all $j$, the following dichotomy holds: Either $p_j\equiv 0$, or the set
\[
N_j = \{v\in \R^d : p_j(v) = 0\}
\]
 has Lebesgue measure $0$.\footnote{ If $n=1$, this is the fundamental theorem of calculus. In the general case, this can be proved by Fubini's theorem and induction on $n$. See e.g.\ \cite{ivanovoverflow} for an alternative elegant proof.} In particular, the set
 \[
 Y = \bigcap_{N_j\not\equiv \R^d} \big(\R^d\setminus N_j\big)
 \]
has full measure. For any $v\in Y$, $\sigma_v(t) = \sigma(tv)$ is an analytic function which is a polynomial of the same degree as $\sigma$ if $\sigma$ is a polynomial, and not a polynomial if $\sigma$ is not a polynomial.

Analogously, the restriction to complex linear maps plays a major role in the context of \cite{voigtlaender2020universal}.
\end{remark}

\begin{remark}\label{remark approximation harmonic polynomials}
As in the complex case, the question whether all harmonic functions in $D$ which extend continuously to $\overline D$ can be approximated by harmonic polynomials depends on the topology of $D$. For a general compact set $K\subseteq\R^d$, the following are equivalent \cite[Theorem 1.3]{MR1342298}:
\begin{enumerate}
\item Every continuous function on $K$ which is harmonic in the interior $K^\circ$ of $K$ can be approximated uniformly by harmonic polynomials.
\item $\R^d\setminus K$ and $\R^d\setminus K^\circ$ are {\em thin} at the same points of $K$.
\end{enumerate}
For a review of thin sets, see e.g.\ \cite[Section 0]{MR1342298} or \cite[Section 5.6]{MR3308615}. In particular, if $D$ is a bounded open set with $C^2$-boundary and $d\geq 3$, every harmonic function on $D$ which extends continuously to $\overline D$ can be approximated uniformly by harmonic polynomials.
\end{remark}

\section{Proof of Theorem \ref{theorem main 2}: Deep residual networks}\label{section proof main 2}

In this section, we establish the approximation properties of deep residual neural networks. The proof of the first claim in Theorem \ref{theorem main 2} is a classical technique of `turning a neural network on its side'.

\begin{proof}
{\bf First claim.} Let 
\[
f_\Theta \in \M_{\sigma,n}, \qquad f_\Theta(x) = \sum_{i=1}^n a_i\,\sigma(\langle w_i,x\rangle + b_i) = \sum_{i=1}^n a_i\,\sigma\left(\sum_{j=1}^d \overline w_{i,j}x_j + b_i\right)
\] 
be a shallow neural network. We represent $f_\Theta$ by a residual neural network as
\begin{align*}
z \mapsto& \xrightarrow{linear} \begin{pmatrix} x\\ 0\end{pmatrix} \xrightarrow{residual} \begin{pmatrix} x + \sigma(0)\\ 0 + \sum_{i=1}^{D_1}a_i\,\sigma(\langle w_i,x\rangle + b_i)\end{pmatrix} = \begin{pmatrix} x \\ \sum_{i=1}^{D_1}a_i\,\sigma(\langle w_i,x\rangle + b_i)\end{pmatrix}\\
	&\xrightarrow{residual} \begin{pmatrix} x \\ \sum_{i=1}^{D_1+D_2}a_i\,\sigma(\langle w_i,x\rangle + b_i)\end{pmatrix} \xrightarrow{residual}\:\dots\: \xrightarrow{residual}\begin{pmatrix} x \\ \sum_{i=1}^{n}a_i\,\sigma(\langle w_i,x\rangle + b_i)\end{pmatrix}\\
		&\xrightarrow{linear} \sum_{i=1}^{D_1}a_i\,\sigma(\langle w_i,x\rangle + b_i)
\end{align*}
with weights 
\[
A^0 = \begin{pmatrix} 1&0&\dots &0\\ 0&1&&\\ \vdots&&\ddots&\vdots\\0&\dots&\dots&1\\0&\dots &\dots &0 \end{pmatrix},\quad b^0 = \begin{pmatrix}0\\\vdots \\0\end{pmatrix}, \quad A^L = \begin{pmatrix} 0&\dots&0&1\end{pmatrix}
\]
and, for $1\leq \ell\leq L-1$, 
\[
A^\ell = \begin{pmatrix}0 &\dots& 0\\ \vdots &\ddots&\vdots\\ 0&\dots &0\\ a_{N_\ell+1} &\dots &a_{N_\ell+ D_\ell}\end{pmatrix}\in \K^{(d+1)\times D_\ell}, \qquad W^\ell = \begin{pmatrix} \overline w_{N_\ell+1,1} &\dots &\overline w_{N_\ell+1,d}\\ \vdots &\ddots &\vdots \\ \overline w_{N_\ell+D_\ell, 1} &\dots &\overline w_{N_\ell+ D_\ell, d}\\ 0 &\dots&0\end{pmatrix}\in \K^{D_\ell\times (d+1)}
\]
where $N_\ell = \sum_{i=1}^{\ell-1} D_i$. The biases are chosen accordingly as $b^\ell = (b_{N_\ell+1}, \dots, b_{N_\ell+D_\ell})$. Of course if $\K=\R$, the complex conjugation has no effect. 

In particular $\overline{\P} = \overline{\M_\sigma}\subseteq \overline{\mathcal N_1}$ since $\M_\sigma\subseteq \mathcal N_1$, and similarly for $\mathcal N_2$. On the other hand, $f\in \mathcal R(d_0, D_1,\dots, D_{L-1})$ is a composition of analytic (vector-valued) functions with globally convergent power series. Consequently, also $f$ can be represented by a globally convergent power series \cite[Section 14.2]{MR2374633}, and by truncating $\mathcal R_{L,\sigma}(d_0, D_1,\dots, D_{L-1})\subseteq\overline\P$ for any choice of architecture $d_0, D_1,\dots, D_{L-1}$. Therefore in particular $\overline {\mathcal N_1}, \,\overline{\mathcal N_2} \subseteq \overline{\P}$.

{\bf Second claim.} If $\sigma$ is not a polynomial, the second claim follows from the first. In the case of polynomial activation, the approximation properties of deep networks are strictly greater than those of networks of fixed depths. We provide a direct proof for the second claim which does not use the first. 

First, consider the case that that $\sigma(z) = z^2$. Again, we construct the ResNet such that the vector $z= (z_1,\dots, z_d)$ is available at all layers. 
We recall that 
\[\showlabel \label{eq products as squares}
z_iz_j = \frac{\sigma(z_i+z_j) - \sigma(z_i-z_j)}4\quad \text{and}\quad z_i = \frac{\sigma(z_i+1) - \sigma(z_i-1)}4.
\]
Let $P(z) = \sum_{j_1+ \dots + j_d\leq m\\, j_i\in \N_0} a_{j_1\dots j_d} z_1^{j_1}\dots z_d^{j_d}$. If $d_0 \geq d+2$ and $D_\ell\geq 2$, we construct the residual representation
\begin{align*}
z &\xrightarrow{linear} \begin{pmatrix} a_{0,\dots,0} + \sum_{j=1}^d a_{0,\dots, 1,\dots,0}z_j\\0\\z\end{pmatrix} 
	\xrightarrow{residual} \begin{pmatrix} a_{0,\dots,0} + \sum_{j=1}^d a_{0,\dots, 1,\dots,0}z_j + a_{1,1,0,\dots,0}z_1^2\\0\\z\end{pmatrix} \\
	&\xrightarrow{residual} \begin{pmatrix} a_{0,\dots,0} + \sum_{j=1}^d a_{0,\dots, 1,\dots,0}z_j + a_{2,0,\dots,0}z_1^2+ a_{1,1,0,\dots,0}z_1z_2\\0\\z\end{pmatrix}\\
	&\xrightarrow{residual}\:\dots\:\xrightarrow{residual} \begin{pmatrix} a_{0,\dots,0} + \sum_{j=1}^d a_{0,\dots, 1,\dots,0}z_j + \sum_{i,j=1}^da_{0,\dots,1,\dots,1,\dots,0}z_iz_j\\0\\z\end{pmatrix}\\
	&\xrightarrow{residual} \begin{pmatrix} a_{0,\dots,0} + \sum_{j=1}^d a_{0,\dots, 1,\dots,0}z_j + \sum_{i,j=1}^da_{0,\dots,1,\dots,1,\dots,0}z_iz_j\\z_1^2\\z\end{pmatrix}\\
	&\xrightarrow{residual} \begin{pmatrix} a_{0,\dots,0} + \sum_{j=1}^d a_{0,\dots, 1,\dots,0}z_j + \sum_{i,j=1}^da_{0,\dots,1,\dots,1,\dots,0}z_iz_j\\z_1^3\\z\end{pmatrix}\\
	&\xrightarrow{residual} \begin{pmatrix} a_{0,\dots,0} + \sum_{j=1}^d a_{0,\dots, 1,\dots,0}z_j + \sum_{i,j=1}^da_{0,\dots,1,\dots,1,\dots,0}z_iz_j+ a_{3,0,\dots,0}z_1^3\\0\\z\end{pmatrix}\\
	&\xrightarrow{residual}\:\dots\:\begin{pmatrix} P(z)\\ 0 \\ z\end{pmatrix} \xrightarrow{linear}P(z).
\end{align*}
The additional zero component is needed to `build up' higher powers of $z$ by using $\sigma$ and 
\[
z_1^{j_1+1}\dots z_d^{j_d}= z_1^{j_1}\dots z_d^{j_d} \cdot z_1
\]
together with \eqref{eq products as squares}, before adding them to the polynomial which is assembled over many layers in the first component. We note that the requirement $D_\ell\geq 2$ was needed in order to execute multiplication -- if multiplication were stretched over multiple layers, an additional `storage' space would be required. Wider networks with wider residual blocks could fit multiple terms of the polynomial at the same time. 

Now consider the case that $\sigma$ is a general analytic function which is not affine-linear. In this case, there exists $z^*\in \K$ such that $\sigma''(z)\neq 0$. As in the proof of Theorem \ref{theorem main 1}, we can approximate $z\mapsto z^2$ to arbitrary accuracy by 
\[
z^2 = \lim_{h\to 0} \frac{\sigma(z^*+hz) - 2\sigma(z^*) + \sigma(z^*-hz)} {h^2\,\sigma''(z^*)},
\]
so we note that we can approximate squares to arbitrary accuracy using three evaluations of $\sigma$ and products to arbitrary accuracy using six evaluations of $\sigma$. However, as we only require squares or the differences of squares for which middle function value $\sigma(z^*)$ cancels out, only $m=4$ neurons are needed rather than $m=6$.

We thus see that $\P \subseteq \overline{\bigcup_{L=1}^\infty \mathcal{R}_{L,\sigma}(d+2, 4,\dots,4)}$. The opposite inclusion holds since every element of the space on the right is analytic.
\end{proof}

While every shallow neural network can be expressed as a deep residual network, the number of parameters required to represent the network increases roughly twofold, since we list $D_\ell\cdot d$ zeros explicitly in the weight matrices $A^\ell, W^\ell$ in every step which are implicit in the shallow neural network. While a shallow neural network is described by $(d+2)n$ parameters, the same function represented as a deep residual network has
\[
(d+1) + \sum_{\ell=1}^{L-1} \big(2 D_\ell(d+1) + D_\ell\big) + d(d+1) + (d+1) = 2(n+1)(d+1) +n
\]
parameters. On the other hand, deep residual networks have a much larger expressive power for this number of parameters, since a single input can pass through multiple non-linear activations $\sigma$.

A similar argument can be made for traditional fully-connected feed forward networks and DenseNets. More details can be found in Appendix \ref{appendix fully connected}.
In the real case, the final proof remains valid under the weaker assumption that $\sigma$ is not affine linear and $\sigma \in C^2$.

\section{Conclusion and Further Directions}

We showed that neural networks with (real or complex) analytic activation function can approximate any function which can be approximated by polynomials, and vice versa. While we focused on the $C^0$-topology in our presentation, the result holds in any $C^k$- or $L^p$- topology. The proofs are simple and only require the approximation of derivatives by difference quotients. In the real case, we reprove the universal approximation theorem by elementary means, reducing it to the Stone-Weierstrass theorem. In the complex case, the situation is more complicated, and results may depend on the topology of the domain of approximation. 

If a function is known to be holomorphic in $d\geq 2$ complex variables, or the domain of approximation does not satisfy the hypotheses of Mergelyan's theorem for $d=1$, then also approximation by neural networks is generally impossible. Similar results are obtained for shallow real `harmonic' neural networks, which utilize orthogonal projections of the data onto a lower-dimensional (but not one-dimensional) linear space.

Finally, we showed that polynomial activation functions are admissible from the perspective of approximation theory for residual networks if the network has a certain minimal width and the depth may be taken arbitrarily large. 

The results as presented above are unsatisfying in several ways.

\begin{enumerate}
\item The results presented here are purely qualitative. No rates of approximation are established under stronger assumptions on the target function in terms of e.g.\ the number of parameters. Some quantitative rates can be found in \cite{https://doi.org/10.48550/arxiv.1810.08033, yarotsky2017error} for target functions which lie either in Sobolev or H\"older spaces, and for approximation in different topologies. However, we emphasize that {\em any} function class which has desirable properties from the perspective of statistical learning theory faces the curse of dimensionality when approximating some function in a too general function class, even in a weak topology \cite{MR4198759}.

\item The proofs above involved the approximation of derivatives by difference quotients. Consequently, the coefficients are large and the representation of the neural network depends critically on cancellation between possibly very large terms. For gradient-based optimizers, such subtle coefficients are hard to find.

\item The proofs cannot be easily modified to include the possibly most popular activation function in practice, the rectified linear united (ReLU) activation $\sigma(x) = \max\{x,0\}$. Analogous results in the case of ReLU activation are presented in Appendix \ref{appendix relu}.
\end{enumerate}

From a negative perspective, our results can be interpreted as a statement that neural networks can approximate exactly the same functions as polynomials. In the real case, this is not surprising since polynomials are dense in the space of continuous functions. Nevertheless, approximation by polynomials of high degree performs poorly even in {\em one-dimensional} interpolation problems, as interpolation with poorly chosen data points leads to high amplitude oscillations at the domain boundary known as the {\em Runge phenomenon} \cite[Chapter 13]{MR3012510}. Universal approximation theorems are therefore unable to explain the superiority of neural networks over other parametrized function classes.

Approximation results for neural networks with bounded weights in a suitable sense can be obtained in suitable model classes adapted to neural networks \cite{MR1237720, bach2017breaking, E:2018abpub, MR4375792, parhi2021banach}, in which it can also be demonstrated that neural networks significantly outperform any linear method in spaces of high dimension \cite[Theorem 6]{MR1237720}.

\section{Acknowledgements}

The authors would like to thank Ron DeVore, Guergana Petrova and Peter Binev for inspiring conversations. JP would also like to acknowledge helpful conversations with Dmitriy Bilyk, Alexey Glazyrin, and Oleksandr Vlasiuk at the SIAM TX-LA Conference.

\appendix

\section{Proof of Lemma \ref{lemma polynomial rotations}} \label{appendix polynomial rotations}

Let $\mathcal{HP}^{d}_n$ be the space of harmonic homogeneous polynomials of degree $n$ in $d$ variables. Harmonic homogeneous polynomials are in one-to-one correspondence with their restriction on the sphere as described below. When we restrict harmonic polynomials to the sphere we obtain what are called {\it spherical harmonics}, which are eigenfunctions of the Laplace-Beltrami operator on the (surface of the) sphere. Similarly there is a unique way to extend a spherical harmonic to a harmonic polynomial on Euclidean space $\mathbb{R}^d$. This and much of what is below are standard facts from harmonic analysis on spheres \cite{MR3060033}.

To prove our lemma we need a result on {\it zonal spherical harmonics}, spherical harmonics which are invariant under rotation with respect to a fixed axis. These functions are unique (up to a multiplicative constant) and can be represented on the sphere as $Z^y_n(x)=C_{n}^\lambda(\langle x,y \rangle)$ where $\lambda=(d-2)/2$ and $C_{n}^{\lambda}(t)$ is the degree $n$ Gegenbauer polynomial (defined below). A well-known observation about these functions from interpolation on the sphere states that zonal spherical harmonics form a basis for the spherical harmonics, and thus by the observation above, the restriction of $\mathcal{HP}^{d}_n$ to the sphere (Theorem \ref{theorem-harm}). 

There are a few additional objects we need to define before proving Lemma \ref{lemma polynomial rotations}. {\it Gegenbauer polynomials} are a class of orthogonal polynomials which play a special role in harmonic analysis on spheres. These functions, denoted $C_{n}^{\lambda}(t)$ define a reproducing kernel in the space $L^2(\mathbb{S}^{d-1})$. 

Before we introduce that relation (known as the {\it addition formula}) we note that spherical harmonics, as eigenfunctions of the Laplace-Beltrami operator, form a basis for $L^2(\mathbb{S}^{d-1})$ and we may index a basis of eigenfunctions in terms of the $n$-th eigenvalue of this operator (its eigenvalues are non-negative and increasing). For each eigenvalue the space of spherical harmonics corresponding to this $n$-th eigenvalue is a linear space $V_n$ of dimension $\dim V_n$ which has a basis $\{Y_{n,1},...,Y_{n,\dim V_n}\}$. The {\it addition formula} tells us that $$C_{n}^{\lambda}(\langle x, y \rangle)=\frac{1}{\dim V_n} \sum_{k=1}^{\dim V_n} Y_{n,k}(x)\overline{Y_{n,k}(y)}.$$

The addition formula shows immediately that the Gegenbauer polynomials are what we call positive definite functions on $\mathbb{S}^{d-1}$  meaning 

$$ \sum_{1\leq i,j \leq k} c_{i}\overline{c_{j}}C_{n}^{\lambda}(\langle x_i,x_j \rangle) \geq 0 $$

\noindent holds for all coefficients $c_1,...,c_k\in \mathbb{C}$ and all $x_1,...,x_k\in \mathbb{S}^{d-1}$. This is equivalent to stating that for any collection of points the matrix obtained by evaluating a Gegenbauer polynomial on the inner product matrix corresponding to our system of points is positive definite. From this the following inequality holds
$$\det\left[C_{n}^{\lambda}(\langle x_i, x_j \rangle) \right]_{i,j=1}^N \geq 0.$$

When this inequality is strict our collection of points is called a {\it fundamental system of points} on the sphere. The content of the following result (which we use in our proof of Lemma \ref{lemma polynomial rotations}) tells us that the zonal spherical harmonics form a basis for the restriction of $\mathcal{HP}^{d}_n$ to the sphere when we evaluate them on a fundamental system.

\begin{theorem}\cite[Thm. 1.3.3]{MR3060033}\label{theorem-harm}
If $\{x_{1},...,x_{N}\}$ is a fundamental system of points on the sphere, then $\{C_n^{\lambda}(\langle \cdot,x_{i} \rangle)\ :\ i=1,2,...,N\},\lambda=(d-2)/2$ is a basis for $\mathcal{HP}^{d}_n$.
\end{theorem}

As a final remark recall that the special orthogonal group $SO(d)$, which may be identified with the set of orthogonal matrices of determinant one, is a locally-compact abelian group. As such, up to multiplicative constant there is unique nontrivial countably additive probability measure on the Borel subsets of $SO(d)$ by Haar's theorem \cite{MR0054173}. This measure is invariant on subsets under action of the group (it is left-translation invariant) and we call this measure the Haar measure of $SO(d)$.

\begin{proof}[Proof of Lemma \ref{lemma polynomial rotations}]
Suppose non-zero $p\in \mathcal{HP}^{d}_n$ is given. Because of the one-to-one correspondence between harmonic polynomials and their restriction to the sphere mentioned above, we will interchangeably refer to $p$ as a function on the sphere $\mathbb{S}^{d-1}$ and one defined on all of $\mathbb{R}^{d}$.

As a function defined on $\mathbb{S}^{d-1}$ we may average $p$ with respect to all rotations which fix the axis $y$, with respect to the Haar measure on the special orthogonal group, obtaining the function \begin{align*} f(x)=\int_{SO(d)} p(O_{y}x)\,\d m(y).\end{align*}

Note that this function is a spherical harmonic, which also extends to a homogeneous polynomial of degree $n$ on $\mathbb{R}^d$, and which additionally is invariant under rotations about the axis $y$. Further this function will take the same value as $p(x)$ at the poles fixed by rotation about axis $y$. Thus so long as $p$ is non-zero at these values we see immediately that $f$ is non-zero also.

Let $V=\mathcal{HP}^{d}_n |_{\mathbb{S}^{d-1}}$ be the linear space of polynomials $\mathcal{HP}^{d}_n$ restricted to the sphere. The space $V$ is a finite dimensional space. We know that for $f$ non-zero, $f(x)=cZ_{n}^y(x)$ for some constant $c\neq 0$ by uniqueness of zonal spherical harmonics of degree $n$, and so $\text{dist}(W,Z_{n}^y(x))=0$. However this holds for every zonal spherical harmonic (this was shown for arbitrary $y$). Theorem \ref{theorem-harm} then shows that if we took $y$ to be each $x_i$ in a fundamental system which avoids the condition $p(x_i)=0$, we then arrive at a basis for $V$ by averaging $p$ over rotations multiple times (since the fundamental system condition is satisfied almost everywhere while the set of zeroes of $p$ have  measure $0$ we can choose such $x_i$ satsifying both these conditions). Each of the basis functions in $V=\mathcal{HP}^{d}_n |_{\mathbb{S}^{d-1}}$ satisfy $\text{dist}(W,Z_{n}^{x_i}(x))=0$ and so $W=V$. 
\end{proof}

\section{Locally uniform approximation and approximation on compact sets}\label{appendix compact open}

There are (at least) two natural ways to study the closure of the function class $\M_\sigma$:
\begin{enumerate}
\item Fix a compact set $K$ and take the closure $\overline{\M_\sigma}^K$ of $\M_\sigma$ in $C^0(K)$.
\item Consider the more global closure of $\M_\sigma$ in the {\em compact-open} topology or {\em topology of compact convergence} \cite[\textsection 46]{MR3728284}:
\[
\overline{\M_\sigma}^{cc} = \big\{f \in C^0(\K^d) : \exists\ (f_n)_{n\in\N} \in \M_\sigma\text{ s.t. } f_n \to f\text{ in }C^0(K) \text{ for all compact }K\subseteq\K^d\big\}.
\]
\end{enumerate}

In this appendix, we compare the different closures. By the nature of the subject, this appendix is more technical than the main text and requires some familiarity with topology, Baire categories, and elliptic partial differential equations. Its content is not needed for the main results of this article, but illustrates and justifies the conceptual difference to the approach taken in \cite{voigtlaender2020universal}.

The two notions of closure are related as follows:

\begin{lemma}
A function $f\in C^0(\K^d)$ satisfies $f\in \overline{\M_\sigma}^{cc}$ if and only if $f|_K\in \overline{\M_\sigma}^K$ for all compact sets $K$.
\end{lemma}

\begin{proof}
The implication $f \in \overline{\M_\sigma}^{cc} \Ra f|_K\in \overline{\M_\sigma}^K$ is trivial. For the opposite implication, note that by definition for all $n\in\N$ there exists $f_n\in \M_\sigma$ such that
\[
\|f_n - f\|_{C^0(\overline{B_n(0)})} < \frac 1n.
\]
Since every compact set $K$ is contained in $B_n(0)$ for sufficiently large $n$, we find that $f_n\to f$ locally uniformly, i.e.\ $f\in \overline{\M_\sigma}^{cc}$.
\end{proof}

In particular, if $f\in \overline{\M_\sigma}^{cc}$ and $K\subseteq \K^d$ is a compact set, then $f|_K\in \overline{\M_\sigma}^K\subseteq C^0(K)$. The opposite question is more subtle: if $f\in \overline{\M_\sigma}^K$ for a fixed compact set $K\subseteq\K^d$, is there $F\in \overline{\M_\sigma}^{cc}(\K^d)$ such that $f = F|_K$?

First consider neural networks with analytic activation function $\sigma:\K\to\K$.

\begin{enumerate}
\item $\K = \R$. In this case, $\overline{\M_\sigma}^{K} = \overline{\P}^K = C^0(K)$ by the Stone-Weierstrass Theorem \cite[Section 15.7]{MR2374633}. Thus $\overline{\M_\sigma}^{cc}$ is the set of functions $f:\K^d\to\K$ such that $f$ is continuous everywhere, the Fr\'echet space $C^0(\K^d)$. 

We can answer the question in the affirmative in the real case due to the Tietze-Urysohn Theorem \cite[Theorem 20.4]{MR924157}:

\begin{theorem*}[Tietze-Urysohn Extension Theorem]
Let $K\subseteq \K^d$ be compact and $f\in C^0(K)$. Then there exists $F \in C_c(\K^d)$ such that $F|_K = f$.
\end{theorem*}

\item $\K = \C$. In this case, we claim that $\overline{\M_\sigma}^{cc}$ is the space of holomorphic functions on $\C^d$. To see this, note that a holomorphic function on $\C^d$ can be expanded into a globally convergent power series. By truncating the series, we see that $f$ can be approximated by polynomials in the uniform topology on every open set. 

On the other hand, assume that $f\in \overline{\M_\sigma}^{cc}$, i.e.\ $f\in \overline{\M_\sigma}^{K}$ for all compact subsets of $\C^d$. Then for every $R>0$, $f$ can be approximated by holomorphic functions arbitrarily well in $C^0(\overline{B_R(0)})$, so $f$ is holomorphic in the interior $B_R(0)$. Consequently $f$ is holomorphic on the entire space $\C^d$.

It is well known that even in one complex dimension, there are holomorphic functions which cannot be extended to the entire complex plane. As an example consider $f:\C\setminus\{0\}\to \C$, $f(z) = z^{-1}$. On the other hand, $f$ can be approximated uniformly on any compact set $K$ for which $\C\setminus K$ is connected by Mergelyan's theorem. 

Thus in the complex case, the answer to our question is {\em negative}.
\end{enumerate}

A statement in the same spirit holds also if $\sigma:\R^k\to\R$ is harmonic. While the closure of $\M_\sigma$ in the topology of locally uniform convergence coincides with the space of harmonic functions on $\R^d$, the closure of $\M_\sigma$ in the topology of uniform convergence on a fixed set $\overline D$ contains functions which are harmonic in $D$, but merely continuous on $\overline D$. In particular, there exists no harmonic extension of $f$ to $\R^d$, and we may miss possible limiting functions if we only consider the topology of locally uniform convergence.

To see that this is true, let $D$ be a bounded $C^2$-domain in $\R^d$. By the Perron method \cite[Section 2.8]{MR1814364}, the Dirichlet problem 
\[
\begin{pde}-\Delta u &=0 &\text{in }D\\ u &= f&\text{on }\partial D\end{pde}
\]
has a solution for every $f\in C^0(D)$ which is continuous on $\overline D$. By Remark \ref{remark approximation harmonic polynomials}, $u$ can be approximated by harmonic polynomials. On the other hand, if $f$ is merely continuous, then $u$ cannot be extended to a harmonic function on $\R^d$, so $u$ is not the restriction of a function $U\in \overline{\M_\sigma}^{cc}$ to $\overline D$, where the closure is taken with respect to the compact-open topology.

In fact, we can define a continuous linear map 
\[
A: \overline{\M_\sigma}^{cc} \to \overline{\M_\sigma}^{D}, \qquad u \mapsto u|_{\overline D} 
\]
and compose further with the trace map $B: C^0(\overline D) \to C^0(\partial D)$. Then $B\circ A(\overline{\M_\sigma}^{cc})\subseteq C^2(\partial D)$ while $B(\overline{\M_\sigma}^{D}) = C^0(\partial D)$. It follows from the well known Banach-Mazurkiewicz  Theorem \cite[Theorem 9]{ashraf2017pathological} that $C^2$ is of first Baire category in $C^0$, i.e.\ it is the union of countably many closed sets which all have empty interior. In this sense, we miss `most' functions which can be approximated in $\overline D$ by studying only the global limiting objects in the topology of locally uniform convergence. A measure-theoretic extension can be found e.g.\ in \cite{MR1260170}.

\section{Classical multi-layer perceptra and DenseNets}\label{appendix fully connected}

\subsection{Fully connected neural networks}

For the sake of completeness, we prove an approximation theorem for classical fully connected neural networks with multiple layers. While the depth of residual networks can reach dozens or hundreds and in extreme cases thousands of layers, the number of layers in a deep neural network is typically more manageable. Nevertheless, both statements in Theorem \ref{theorem main 2} have analogues for deep fully connected networks. 

A fully connected deep neural network is defined as follows:
\begin{itemize}
\item For a given input $z\in \overline D\subseteq \K^d$, designate $\hat z^0 = z$ and $d_0 = d$.
\item For $\ell \in \{1, \dots, L\}$, let $d_\ell\in \N$ and set 
\[
z^\ell :=  A^\ell \hat z^{\ell-1} + b^\ell\quad\text{and}\quad \hat z^\ell = \sigma( z^{\ell}),
\]
where $A^\ell \in \K^{d_\ell\times d_{\ell-1}}$ is a linear map which takes $\hat z^{\ell-1}\in \K^{d_{\ell-1}}$ to $z^\ell \in \K^{d_\ell}$, and $b^\ell\in \K^{d_\ell}$. The function $\sigma$ is applied to the vector $z^\ell$ coordinate-wise.
\end{itemize}

We designate the parameters (or ``weights'') of the deep neural network by 
\[
\Theta =\big(A^1, b^1, \dots, A^L, b^L\big) \in \K^{d_1\times d}\times \K^{d_1} \times \dots \times \K^{1\times d_{\ell-1}} \times \K
\]
and set $f_\Theta(z) = z^L \in \K^{d_L}=\K$. A fully connected neural network is described by the choice of activation function $\sigma$, depth $L$ and width of the layers $d_0, \dots, d_L$ (where $d_0=d$ and $d_L = 1$ in our case). We designate the class of fully connected neural networks with such architecture by $\mathcal{FNN}_{\sigma, L} (d_0,\dots, d_L)$.

The following is the analogue of Theorem \ref{theorem main 2} for deep fully connected neural networks.

\begin{theorem}\label{theorem fully connected}
Let $d_\ell = d + 1+ m_\ell$ for $m_\ell \in \N$ and $1\leq \ell \leq L-1$. Set $n:= \sum_{\ell=1}^{L-1}m_\ell$. Assume that $\sigma$ is not a constant function. Then for every compact set $K\subset\K^d$ we have
\[
\M_{\sigma, n} \subseteq \overline{ \mathcal{FNN}_{\sigma,L}(d, d_1,\dots, d_{L-1}, d_L)}
\]
for the closure in $C^0(K)$.
\end{theorem}

\begin{proof}
Since $\sigma$ is not a constant function, there exists $z^*$ in $\K$ such that $\sigma'(z^*)\neq 0$. We write $\sigma(z^*+\eps z) = c_0 + c_1\eps z + O\big((\eps z)^2\big)$ for small $\eps$ and $c_0=\sigma(z^*)$ and $c_1 = \sigma'(z^*)$. Since $z\in K$, there exists $R>0$ such that $|z|\leq R$, so the error term is uniformly small in $z$. We initially ignore the quadratic error term and note that any affine function of $z$ can be written as an affine function of $c_0 + c_1\eps z$. By an abuse of notation we denote by $c_0, z^*$ also the vectors with identical entries in $\K^d$. With the notations
\[
N_\ell = \sum_{i=1}^{\ell-1} m_\ell, \qquad W_\ell = \begin{pmatrix}\overline w_{N_\ell+1, 1} &\dots & \overline w_{N_\ell+1,d}\\ \vdots&\ddots&\vdots\\ \overline w_{N_\ell+m_\ell, 1}& \dots &\overline w_{N_\ell+m_\ell,d}\end{pmatrix} \in \R^{m_\ell\times d}, \quad a_\ell = (a_{N_\ell+1}, \dots, a_{N_\ell + m_\ell})
\]
we represent 
\begin{align*}
z &\xrightarrow{linear} \begin{pmatrix}0\\ W_1z+b_1\\ z^* + \eps z\end{pmatrix} \xrightarrow[\approx]{\sigma} \begin{pmatrix}\sigma(0)\\ \sigma(W_1z+b_1)\\ c_0+c_1\eps z\end{pmatrix} \xrightarrow[\approx]{linear} \begin{pmatrix} z^* + \eps\,a_1^T \sigma(W_1z+b_1)\\ W_2z+b_2\\ z^*+\eps z\end{pmatrix}\\
	&\xrightarrow[\approx]{\sigma}\dots\xrightarrow[\approx]{\sigma} \begin{pmatrix}c_0 + c_1\eps \sum_{i=1}^{n-m_{L-1}}a_i \sigma(\langle w_i,z\rangle+b_i)\\ \sigma(W_{L-1}z+b_{L-1})\\ z^* + \eps z\end{pmatrix} \xrightarrow[\approx]{linear}\sum_{i=1}^{n}a_i \sigma(\langle w_i,z\rangle+b_i).
\end{align*}
The approximation error can be made arbitrarily small without increasing the number of parameters, by taking $\eps$ sufficiently small. When writing an affine function of $z$ as an affine function of $c_0+ c_1\eps z$, this can lead to potentially poorly conditioned linear maps, which may cause numerical instability in practice. Taking the limit $\eps\to 0$, we see that any shallow neural network can be approximated arbitrarily well by a deeper neural network with sufficient width and length.
\end{proof}

\begin{remark}
The number of weights of a deep neural network with architecture $d_1,\dots, d_L$ is
\[
N = \sum_{\ell=1}^L \big(d_{\ell-1}\cdot d_\ell + d_\ell\big).
\]
Assume that $L:= n/d$ is an integer. Then we can construct two networks to approximate $f\in \M_{\sigma,n}$:
\begin{enumerate}
\item A network of depth $n$ and width $d+2$. The number of weights is $N \sim n d^2$.
\item A network of depth $L$ and width $2d+1$. The number of weights is $N\sim L(2d)^2\sim 4n/d$.
\end{enumerate}
Thus the number of weights increases by a factor $d$ if we approximate $f$ by a thin deep network and by a factor $4$ if we use a network which is both wide and deep instead.
\end{remark}

Thus in a sense, we can argue that anything which can be achieved by a shallow neural network can also be achieved by a deep neural network with a comparable number of parameters, if the width of the deep network is not too small. Since deep neural networks with analytic activation are analytic, we find the following.

\begin{corollary}\label{corollary approximation deep}
Consider the classes of neural networks 
\[
\mathcal{NN}_1 = \bigcup_{L=1}^\infty \mathcal{FNN}_{\sigma,L}(d, d+1+ m, \dots, d+1+m, 1)
\]
for any fixed $m\in\N$ and 
\[
\mathcal{NN}_2 = \bigcup_{m=1}^\infty \mathcal{FNN}_{\sigma,L}(d, d+1+ m, \dots, d+1+m, 1)
\]
for fixed $L$. Then for any compact set $K\subseteq \K^d$ we have $\overline{\mathcal{NN}_1} = \overline{\mathcal{NN}_2} = \overline{\P}$ in $C^0(K)$.
\end{corollary}

\subsection{DenseNets}
A DenseNet \cite{iandola2014densenet} is a modified neural network structure in which the internal state of the $\ell$-th layer is computed linearly from the state of the network at all previous layers $0, \dots, \ell-1$, rather than just the previous state. Namely, consider the following structure. 

\begin{itemize}
\item For a given input $z\in \overline D\subseteq \K^d$, designate $\hat z^0 = z$ and $d_0 = d$.
\item For $\ell \in \{1, \dots, L\}$, let $d_\ell\in \N$ the width of the $\ell$-th state,
\[
n_\ell = \sum_{i=0}^{\ell-1} d_i, \qquad A^\ell \in \K^{d_\ell\times n_\ell}, \quad b^\ell\in \K^{d_\ell}
\] 
and 
\[
z^\ell = A^\ell \begin{pmatrix} \hat z^0\\ \vdots\\ \hat z^{\ell-1}\end{pmatrix} + b^\ell, \qquad \hat z^\ell = \sigma(\hat z^\ell).
\]
\[
z^\ell :=  A^\ell \hat z^{\ell-1} + b^\ell\quad\text{and}\quad \hat z^\ell = \sigma( z^{\ell}),
\]
The function $\sigma$ is applied to the vector $z^\ell$ coordinate-wise.
\item The output of the network is $f(z) = z^L$, where we suppressed the dependence on the input $z$ and the weights $(A^0, b^0, \dots, A^L, b^L)$ in the notation.
\end{itemize}

We denote the class of DenseNets with activation $\sigma$ and widths $(d_0, d_1, \dots, d_{L-1}, d_L)$ by $\mathcal{D}_{\sigma, L}(d_0, d_1, \dots, d_{L-1}, d_L)$. As usual, $d_0 = d$ and $d_L = 1$ are fixed by the problem statement.
Since a DenseNet can access all previous states (including the input) at all layers, every neural network with a single hidden layer and activation function $\sigma$ can be represented exactly by a sufficiently large DenseNet. No restriction to compact sets or closure operation is required, and there is no need to invert potentially ill-conditioned matrices. The same result holds for deep neural networks, since the previous layer can be accessed.

\begin{theorem}\label{theorem densenet}
Set $n=  \sum_{\ell=1}^{L-1}d_\ell = n_L -d$. Then we have
\[
\M_{\sigma, n} \:\cup\: \mathcal{FNN}_{\sigma,L}(d, d_1,\dots, d_{L-1}, 1) \quad\subseteq\quad \mathcal{D}_{\sigma, L}(d, d_1, \dots, d_{L-1}, 1).
\]
\end{theorem}

In particular, if $\sigma$ is analytic and not polynomial/affine, an analogue of Corollary \ref{corollary approximation deep} holds for DenseNets. Note, however, that the number of parameters for DenseNets with comparable widths is significantly larger, since the weights $A^\ell \n \K^{d_\ell\times n_\ell}$ used in the construction of Theorem \ref{theorem densenet} have block structures which list many zeros explicitly, as only the first/previous layer is accessed.

\section{A few words on ReLU activation}\label{appendix relu}

Our presentation on neural networks does not apply to one of the most popular activations in practice: $\sigma:\R\to\R$ given by $\sigma(x) = \mathrm{ReLU}(x) = \max\{x,0\}$. For the sake of completeness, we sketch how similar results can be obtained in this situation by different means. Almost all activation functions fall into one of these categories: real analytic, or ReLU-like (e.g.\ $x\mapsto \max\{x,\eps x\}$ for $\eps \in[0,1)$).

\begin{example}[Shallow neural networks]
The direct approximation for shallow ReLU networks can also be proved in elementary fashion since a $C^2$-function $f:\R\to\R$ can be represented as
\[\label{eq relu representation formula}
f(x) = f(a)\,\mathrm{ReLU}(1) + f'(a) \,\mathrm{ReLU}(x-a) + \int_a^\infty f''(t)\,\mathrm{ReLU}(x-t)\,\dt
\]
for $x>a$, and the integral can be discretized by a Riemann sum uniformly for $x\in [a,b]$. In particular, the function $x\mapsto x^m$ can be approximated by a shallow ReLU network on any bounded interval. The third step of the proof of Theorem \ref{theorem main 1} still applies, and we we can approximate $x_1^{m_1}\dots x_d^{m_d}$ for any $m_1,\dots, m_d\in \N$ by approximating suitable derivatives of $(h_1x_1+ \dots + h_dx^d)^{m_1+\dots+m_d}$.

We have thus proved the direct approximation theorem for ReLU networks in an entirely elementary fashion.

Equation \eqref{eq relu representation formula} can be established using the fundamental theorem of calculus and integrating by parts or using Fubini's theorem to exchange the order of integration. A faster, but conceptually more involved, proof utilizes the fact that the second derivative of the ReLU function is the measure localized at zero which measures the magnitude of the jump in the first derivative (i.e.\ a Dirac $\delta$).
\end{example}

\begin{example}[Deep residual networks]
As with fully connected networks, it is possible to deduce the universal approximation theorem for deep residual networks from that for shallow networks. We go a different route here, which uses the deep structure in a more interesting way. The argument is due to Boris Hanin \cite{hanin2019universal}. 

Recall that any function $g:\R^d\to\R$ with a uniformly bounded Hessian can be written as the difference of non-negative strongly convex functions:
\[
g(x) = \left(g(x) + \frac\lambda2\,\|x\|^2 + c\right) - \frac\lambda2 \|x\|^2 - c = f_1(x) - f_2(x).
\]
The convexity property holds since the Hessians of $f_1, f_2$ satisfy
\[
D^2 f_1 = \lambda I_{d\times d} + D^2g\geq I_{d\times d}, \qquad D^2f_2 = \lambda\,I_{d\times d}\geq I_{d\times d}
\] 
if $\lambda>0$ is large enough. Since strongly convex functions are bounded from below, we can choose $c>0$ large enough to make $f_1, f_2$ non-negative.

In particular, since $C^\infty_c(\R^d)\cap C^0(K)$ is dense in $C^0(K)$ for any compact $K\subseteq\R^d$,\footnote{  To prove this, we use the Tietze-Urysohn extension \cite[Theorem 20.4]{MR924157} and convolution as in Theorem \cite[Theorem 4.23]{MR2759829}.} we find that the class of {\em dc functions} (functions which are the difference of two convex functions) is dense in $C^0(K)$. It therefore suffices to show that every dc function can be approximated arbitrarily well by deep residual ReLU networks.

Recall furthermore that a convex function $f:\R^d\to\R$ can be written as the supremum of affine linear functions:
\[
f(x) = \sup\{\langle w,x\rangle +b : (w,b)\in A_f\}
\]
for some set $A_f\subseteq \R^d\times \R$, using e.g.\ the convex conjugate and the Fenchel-Moreau Theorem \cite[Theorem 1.11]{MR2759829}. On any compact set $K\subseteq \R^d$, $f$ can be approximated uniformly by the maximum of finitely many linear functions
\[\showlabel\label{eq finite convex}
f_L(x) = \max\{\langle w_i,x\rangle +b_i : 1\leq i\leq L\}.
\]
We show that it is possible to represent $f_L$ exactly by a deep residual network of depth $L+1$ and width $d+1$, and to represent $f_L^{(1)} - f_{L'}^{(2)}$ for two different functions of the form \eqref{eq finite convex} by a deep residual network of depth $\max\{L,L'\}$ and width $d+2$. 

First, consider the case in which $f_L$ is convex and non-negative. Write 
\[
g_k(x) = \max\{0, \langle w_1,x\rangle + b_1, \dots, \langle w_k,x\rangle + b_k\}
\]
and recall that $\max\{\alpha,\beta\} = \alpha + \sigma(\beta-\alpha)$ for all $\alpha, \beta\in \R$, as well as 
\[
\max\{\alpha_1,\dots, \alpha_L\} = \max\big\{\max\{\alpha_1,\dots,\alpha_{k-1}\}, \alpha_k\}.
\]
On this basis, we construct a residual network as
\begin{align*}
x\xrightarrow{linear} \begin{pmatrix} 0\\x\end{pmatrix}&\xrightarrow{residual}\begin{pmatrix}0 + \sigma(\langle w_1, x\rangle + b_1)\\x + \sigma(0)\end{pmatrix} = \begin{pmatrix} \max\{0, \langle w_1,x\rangle + b_1\} \\ x\end{pmatrix} = \begin{pmatrix} g_1(x)\\ x\end{pmatrix}\\
	&\xrightarrow{residual}\begin{pmatrix}  g_1(x)+ \sigma \big(\langle w_2,x\rangle + b_2 - g_1(x)\big) \\ x\end{pmatrix} = \begin{pmatrix} g_2(x)\\x\end{pmatrix}\\
	&\xrightarrow{residual}\dots \xrightarrow{residual} \begin{pmatrix}g_L(x)\\ x\end{pmatrix} \xrightarrow{linear} g_L(x).
\end{align*}
If $f = f_L^{(1)} - f_L^{(2)}$ and the width of the network is increased to $d+2$, then two convex functions can be generated simultaneously, and the final linear layer can be used to express their difference.

Thus the class of deep residual networks of width at least $d+2$ is dense in the class of dc functions with respect to the $L^\infty_{loc}$-topology, and thus in $C^0(K)$ for any compact set $K\subseteq \R^d$.

Note that this is true for a particularly simple class of residual networks compared to the general form \eqref{eq resnet}, namely $D_\ell=1$ and $A^\ell = (1,0,\dots,0)$ for all $1\leq\ell \leq L$ in the convex case and 
\[
D_\ell = 2, \qquad A^\ell = \begin{pmatrix}1&0 & 0&\dots &0\\ 0&1&0&\dots &0\end{pmatrix}^T
\]
in the general case.
\end{example}

\begin{remark}\label{remark maximum function}
We note that increasing the width $d_0$ of the residual network (in unison with the residual blocks) drastically decreases the required depth. While with one additional layer we are only able to take the maximum of one additional linear function per layer, due to the identity
\[
\max\{y_1,\dots, y_{2k}\} = \max\big\{\max\{y_1,y_2\}, \dots, \max\{y_{2k-1},y_{2k}\}\big\},
\]
with a residual network of width $d_0 + k$, it is possible to take the maximum of $k$ linear functions in $\lceil \log_2(k)\rceil$ layers. 
\end{remark}

\begin{example}[Fully connected deep neural networks]
We can argue that any shallow neural network with ReLU activation can be represented {\em exactly} by a deeper network with sufficient width rather than just approximated at the cost of a marginally larger width. The arguments are a simpler version of those in Appendix \ref{appendix fully connected} and do not involve the inversion of possibly ill-conditioned linear maps with a small parameter $\eps>0$. Moreover, they apply on the entire quadrant $Q_+ = \{x\in \R^d : x_i>0 \:\forall\ 1\leq i\leq d\}$.

{\bf Claim:} {\em Let $K\subseteq \K^d$ compact and $f\in \M_{\sigma,n}$ for $\sigma = \mathrm{ReLU}$. Then $f\in \mathcal{FNN}_{\sigma,L}(d,d_1,\dots, d_{L-1},1)$ if 
\begin{enumerate}
\item $d_\ell \geq d+2$ for all $1\leq \ell\leq L-1$ and
\item $\sum_{\ell=1}^L (d_\ell - d-2) = n$.
\end{enumerate}
}

{\bf Proof of claim:} Let $K\subseteq \K^d$ be a compact set and $f(x) = \sum_{i=1}^n a_i\sigma(\langle w_i,x\rangle + b_i)$ an element of $\M_{\sigma,n}$. Due to the compactness of $K$, there exist a scaling and translation $x\mapsto \alpha x+\beta$ such that $(\alpha x+\beta)_i\geq 0$ for all $i=1,\dots, d$. We note that any affine linear function of $x$ can also be expressed as an affine linear function of $\hat x = \alpha x+\beta$.

We introduce the notations
\[
m_i = d_i - (d+2), \quad N_\ell = \sum_{i=1}^{\ell-1} m_i, \qquad W_\ell = \begin{pmatrix}w_{N_\ell+1, 1} &\dots &w_{N_\ell+1,d}\\ \vdots&\ddots&\vdots\\ w_{N_\ell+m_\ell, 1}& \dots &w_{N_\ell+m_\ell,d}\end{pmatrix} \in \R^{m_\ell\times d}
\]
and 
\[
I^+_\ell = \big\{i \in \{1,\dots, N_\ell\} : a_{i} \geq 0\big\}, \qquad I^-_\ell = \{0,\dots, N_\ell\}\setminus I^+_\ell.
\]
Using that $\sigma(z)\geq 0$ for all $z\in\R$ and $\sigma(z) = z$ for all $z\geq 0$, we may represent $f$ by a deeper network as
\begin{align*}
x&\xrightarrow{linear} \begin{pmatrix}0\\0\\ W_1 + b_1\\ \alpha x+\beta\end{pmatrix} \xrightarrow{\sigma} \begin{pmatrix}\sigma(0)\\ \sigma(0) \\ \sigma(W_1x + b_1)\\ \sigma(\alpha x+\beta)\end{pmatrix} = \begin{pmatrix} 0\\ 0\\ \sigma(W_1x+b_1)\\ \alpha x+\beta\end{pmatrix}
	 \xrightarrow{linear} \begin{pmatrix} \sum_{i\in I^+_1} a_i \sigma(\langle w_i,x\rangle + b_i)\\ -\sum_{i\in I^-_1} a_i \sigma(\langle w_i,x\rangle + b_i)\\ W_2x+b_2\\ \alpha x+\beta\end{pmatrix}\\&
	 \xrightarrow{\sigma} \begin{pmatrix} \sigma\left(\sum_{i\in I^+_1} a_i \sigma(\langle w_i,x\rangle + b_i)\right)\\ \sigma\left(-\sum_{i\in I^-_1} a_i \sigma(\langle w_i,x\rangle + b_i)\right)\\ \sigma\big(W_2x+b_2\big)\\ \sigma(\alpha x+\beta)\end{pmatrix}
	= \begin{pmatrix} \sum_{i\in I^+_1} a_i \sigma(\langle w_i,x\rangle + b_i)\\ -\sum_{i\in I^-_1} a_i \sigma(\langle w_i,x\rangle + b_i)\\ \sigma(W_2x+b_2)\\ \alpha x+\beta\end{pmatrix}\xrightarrow{linear}\:\dots\:\xrightarrow{\sigma}\\ 
	&\xrightarrow{\sigma}
	 \begin{pmatrix} \sum_{i\in I^+_{L-1}} a_i \sigma(\langle w_i,x\rangle + b_i)\\ -\sum_{i\in I^-_{L-1}} a_i \sigma(\langle w_i,x\rangle + b_i)\\ \sigma(W_2x+b_2)\\ \alpha x+\beta\end{pmatrix}
	 \xrightarrow{linear}\sum_{i\in I^+_{L-1}} a_i \sigma(\langle w_i,x\rangle + b_i) + \sum_{i\in I^-_{L-1}} a_i \sigma(\langle w_i,x\rangle + b_i)\\& = \sum_{i=1}^n a_i\,\sigma(\langle w_i,x\rangle + b_i) = f(x).
\end{align*}
We note that also the convex analytic method of proof generalizes to multi-layer perceptra. Let $\phi_1,\dots, \phi_L:\R^d\to\R$ be linear maps. We construct
\begin{align*}
x&\xrightarrow{linear} \begin{pmatrix}0\\\phi_1(x)-\phi_2(x)\\ \alpha x+\beta\end{pmatrix} \xrightarrow{\sigma} \begin{pmatrix}0\\\max\{\phi_1(x) - \phi_2(x),0\} \\ \alpha x+\beta\end{pmatrix}\\
	&\xrightarrow{linear}\begin{pmatrix}\phi_2(x) + \max\{\phi_1(x) - \phi_2(x),0\}\\ \phi_3(x)  \\ \alpha x+\beta\end{pmatrix} = \begin{pmatrix} \max\{\phi_1(x), \phi_2(x) \}\\ \phi_3(x)  \\ \alpha x+\beta\end{pmatrix}\\
	&\xrightarrow{linear}\begin{pmatrix} \max\{\phi_1(x), \phi_2(x) \}\\ \phi_3(x) - \max\{\phi_1(x), \phi_2(x) \}  \\ \alpha x+\beta\end{pmatrix} \xrightarrow {\sigma} \begin{pmatrix} \max\{\phi_1(x), \phi_2(x) \}\\ \sigma\big(\phi_3(x) - \max\{\phi_1(x), \phi_2(x) \}\big)  \\ \alpha x+\beta\end{pmatrix}\\
	&\xrightarrow{linear}\begin{pmatrix} \max\{\phi_1(x), \phi_2(x), \phi_3(x) \}\\ \phi_4(x)  \\ \alpha x+\beta\end{pmatrix}
	\xrightarrow {\sigma}\:\dots\:\xrightarrow {\sigma} \begin{pmatrix} \max\{\phi_1(x), \dots, \phi_{L-1} \}\\ \sigma\big(\phi_L(x) - \max\{\phi_1(x), \dots\phi_{L-1}(x) \}\big)  \\ \alpha x+\beta\end{pmatrix}\\
	& \xrightarrow{linear} \max\big\{\phi_1(x),\dots,\phi_L(x)\big\}.
\end{align*}
As the composition of linear maps is linear, this network can be represented by a standard neural network architecture. Increasing the width of the network may drastically reduce the required depth as in Remark \ref{remark maximum function}.

\end{example}


\newcommand{\etalchar}[1]{$^{#1}$}

\end{document}